\documentclass[12pt]{article}
\usepackage{framed}
\RequirePackage{kpfonts}
\RequirePackage[sf,bf,small,raggedright]{titlesec}
\usepackage{textcomp}
\usepackage{amssymb}
\usepackage{bbm}
\usepackage{fancyhdr}
\usepackage{amsmath}
\usepackage{amsbsy,amsthm}
\usepackage{amscd}
\usepackage{latexsym}
\usepackage{graphicx}   
\usepackage{pdfsync}
\usepackage{blkarray}
\usepackage{multirow}

\usepackage[colorlinks=true,breaklinks=true,bookmarks=true,urlcolor=blue,
citecolor=blue,linkcolor=blue,bookmarksopen=false,draft=false]{hyperref}

\usepackage[utf8]{inputenc}
\usepackage{amsmath,amsthm,amssymb,mathtools,enumerate,verbatim,url}
\usepackage[algoruled,linesnumbered]{algorithm2e}

\usepackage[numbers]{natbib}
\def\BIBand{and}%
\bibpunct[, ]{[}{]}{,}{n}{}{,}%

\theoremstyle{plain}
\newtheorem{theorem}{Theorem}
\newtheorem{lemma}[theorem]{Lemma}
\newtheorem{proposition}[theorem]{Proposition}
\newtheorem{corollary}[theorem]{Corollary}

\theoremstyle{definition}
\newtheorem*{definition}{Definition}

\theoremstyle{remark}

\newcommand{\eps}{\varepsilon}
\newcommand{\E}{\mathbf{E}}
\renewcommand{\P}{\mathbf{P}}
\newcommand{\p}[1]{\P\left\{{#1}\right\}}

\newcommand{\muhat}{\widehat{\mu}}
\DeclareMathOperator{\kl}{kl}
\renewcommand{\epsilon}{\varepsilon}
\newcommand{\mubar}{\overline{\mu}}

\DeclareMathOperator*{\argmax}{arg\,max}

\def\EMAIL#1{\href{mailto:#1}{#1}}

\author{G\'abor Lugosi\thanks{Department of Economics and Business,
		Pompeu  Fabra University, Barcelona, Spain}
	 \thanks{ICREA, Pg. Llu\'is Companys 23, 08010 Barcelona, Spain}
	 \thanks{Barcelona Graduate School of Economics,
		\EMAIL{gabor.lugosi@gmail.com}}
	\and Abbas Mehrabian\thanks{McGill University,
		\EMAIL{abbas.mehrabian@gmail.com}}}

\title{Multiplayer bandits \\ without observing collision information\thanks{To appear in \textit{Mathematics of Operations Research}}}
\begin{document}
	
	\maketitle
	
	\begin{abstract}
		We study multiplayer stochastic multi-armed bandit problems
in which the players cannot communicate
and if two or more players pull the same arm, a collision occurs and the involved players receive zero reward.
We consider two feedback models:
a model in which the players can observe whether a collision has
occurred and a more difficult setup when no collision information is available.
We give the first theoretical guarantees for the second model:
an algorithm with a logarithmic regret
and an algorithm with a square-root regret that does not depend on the gaps between the means.
For the first model, we give the first square-root regret bounds that do not depend on the gaps.
Building on these ideas, we also give an algorithm for reaching  approximate Nash equilibria quickly in stochastic anti-coordination games.
	\end{abstract}
	
	
	MSC2020 subject classification: Primary: 68Q32; Secondary: 62L12, 68W15, 91A15.

\section{Introduction.}
The stochastic multi-armed bandit problem is a well-studied problem of machine learning.
Consider an agent that has to choose among several actions in each round of a game.
To each action $i$ is associated a real-valued parameter $\mu_i$.
Whenever the player performs the $i$th action,
she receives a random reward with mean $\mu_i$.
If the player knew the means associated to the actions before starting
the game, 
she would play an action with the highest mean during all rounds.
The problem is to design a strategy for the player to maximize her reward in the setting where she does not know the means.
The \emph{regret} of the strategy is the difference between the accumulated rewards in the two scenarios.

This problem encapsulates the well-known exploration/exploitation trade-off: 
the player never learns the means exactly, but she can estimate them.
As the game proceeds, she learns that some  actions {probably} have better means, so she can exploit these actions to obtain a better reward, but at the same time she has to explore other actions as well, since they {might} have higher means.
Traditionally, actions are called ``arms'' 
and ``pulling an arm'' refers to performing an action.
See~\citet*{Slivkins_survey,torcsababook} for recent monographs on stochastic multi-armed bandits.

We study a multiplayer version of this game, in which each player
pulls an arm in each round, and if two or more players pull the same 
arm, a \emph{collision} occurs and all players pulling that arm receive zero reward.
The players' goal is to maximize the collective received reward.

One application for this model is opportunistic spectrum access with multiple users in a cognitive radio network: we have a radio network with several channels (corresponding to the arms) that have been purchased by primary users. There are also secondary users (the players) that can try to use these channels during the rounds when the primary users are not transmitting.
Successfully using a channel to transmit a message means a unit reward, and not transmitting means zero reward.
If more than one secondary users try to use the same channel in the
same round, a collision occurs and none of them can transmit.
If a unique secondary user tries to use a channel, she will succeed if the primary user owning that channel happens to be idle in that round, which happens with a certain probability. Thus, the reward of the secondary user is a Bernoulli random variable whose mean depends on the activity of the corresponding primary user and whether other secondary users have tried to use the same channel.
See~\citet*[Section I.D]{liuzhao} for other applications.

One may consider (at least) two possible feedback models.
In the first model, whenever a player pulls an arm, she observes whether a collision has occurred on that arm and receives a reward.
In the second model, the player just receives a reward without observing whether a collision has occurred. Of course, if the reward is positive, she can infer that no collision has occurred.
But if the reward is zero, she cannot infer if a collision has occurred.

Our main contributions are as follows.

\begin{enumerate}
	\item
	We offer the first theoretical guarantees for the second model, where the players do not observe collision information.
	We propose an algorithm with a logarithmic regret (in terms of the number of rounds),
	and we also give an algorithm with a sublinear regret that does not depend on the gaps between the means.
	
	\item
	For the first model, in which the players observe collision information, we prove the first sublinear regret bound that does not depend on the gaps between the means.
	
	\item
	One may also view this setup as a stochastic anti-coordination game.
	Using the algorithmic ideas introduced here, we give an algorithm for reaching an approximate Nash equilibrium quickly in such games.
\end{enumerate}

\subsection{Models and results.}
Let $K>1$ be a positive integer and let $\mu_1,\dots,\mu_K$ be nonnegative numbers corresponding to the arm means.
Let
$Y_{i,t}$ be the reward of arm $i$ in round $t$, so the $\{Y_{i,t}\}_{t=1}^{\infty}$
are independent and identically distributed (i.i.d.) and $\E Y_{i,t} = \mu_i$.
We may assume, by relabeling the arms if necessary, that $\mu_1\geq \dots \geq \mu_K$.
The players are of course unaware of this labeling.

For a positive integer $n$, we denote $[n]\coloneqq \{1,\dots,n\}$.
A set of $m>1$ players play the following game for $T>0$ rounds: in
each round $t=1,\dots,T$, player $j$
chooses an arm $A_j(t)\in [K]$.
Let $C_i(t)\in\{0,1\}$ be the collision indicator for arm $i$ in round $t$, that is,
$C_i(t)=1$ if and only if there exist distinct $j,j'$ with $A_j(t)=A_{j'}(t)=i$.
In round $t$, player $j$ receives reward 
\begin{equation}
r_j(t) = Y_{A_j(t),t} (1 - C_{A_j(t)}(t)).
\label{rewardeq}
\end{equation}

We will also consider a stronger feedback model, in which each player $j$ also observes $C_{A_j(t)}(t)$ in each round $t$; 
this is called ``the model with collision information.''

The \emph{regret} of a strategy is defined as 
\begin{equation}
\label{regret_def}
\textnormal{Regret}=
T \sum_{i\in[m]} \mu_i
-
\sum_{t\in[T]}
\sum_{j\in[m]}
\mu_{A_j(t)} (1 - C_{A_j(t)}(t)).
\end{equation}

Note that Regret is a random variable (since the strategy can randomize hence $A_j(t)$ can be random) and we will bound its expected value. Bounds that hold with high probability can also be derived from our proofs.

To simplify the statements and proofs of our main theorems, we make three additional assumptions, which can be relaxed at the expense of getting worse bounds, as discussed in Section~\ref{sec:relaxing}.

\begin{description}
	\item{Assumption 1.}
	$K\geq m$: there are at least as many arms as players.
	\item{Assumption 2.}
	$Y_{A_j(t),t}$ is supported on $[0, 1]$ so the means $\mu_i$ and the
	rewards $r_j(t)$ are also in $[0,1]$.
	\item{Assumption 3.}
	All players know the values of both $T$ and $m$.
\end{description}
\label{assumptions}

Note that we assume no communication between the players, and our algorithms are totally distributed.
Moreover, in each particular setting, all players play the same algorithm.
All of our algorithms are explicit, simple, and efficient.

We can now state our main theorems.
Let $\Delta \coloneqq \mu_m - \mu_{m+1}$.
All the following results correspond to the weak feedback model (i.e.,
no collision information), unless stated otherwise.
Certainly, any regret upper bound for this model automatically carries over to the stronger feedback model as well.

\begin{theorem}
	\label{thm:firstmain}
	There is an algorithm with expected regret 
	$O ( m K \log (T) / \Delta^2) $.
\end{theorem}

In this theorem and throughout, the notation $f=O(g)$ means there exists an \emph{absolute constant} $C$ such that for {\em all} admissible parameters,
$f \leq C g$. 

A shortcoming of Theorem~\ref{thm:firstmain} is that it gives a vacuous bound if
$\Delta=0$.
Moreover, one may wonder if, as in the single player case, a regret of the form $\sqrt T$ is possible that is independent of the specific instance.
The following theorem shows this is possible, under some weak assumptions.
Let $\Delta'\coloneqq \min \{\mu_m - \mu_i: \mu_i < \mu_m\}$.
Observe that $\Delta'\geq\Delta$, 
and that $\Delta'$ is positive and well-defined unless $\mu_m=\mu_{m+1}=\dots=\mu_K$ (in this case we define $\Delta'=0$).

\begin{theorem}
	\label{thm:secondmain}
	(a) Suppose all players know a lower bound $\mu$ for $\mu_m$.
	Then there is an algorithm with expected regret 
	$O(K^2 m \log^2 (T)/\mu + Km 
	\min \{\sqrt{T \log T}, \log(T)/\Delta'\})$.
	
	(b) For the stronger feedback model, in which the players observe the collision information, there is an algorithm with expected regret 
	\[
	O(K^2 m \log^2 (T) + Km 
	\min \{\sqrt{T \log T}, \log(T)/\Delta'\}) = O(K^2m\sqrt {T\log T}).
	\]
	
	(c) Suppose each player has the option of leaving the game at any point; that is, she can choose not to pull from some round onward (if a player leaves the game, we assume that she collects reward 0 for the rest of the game).
	Then, there exists an algorithm with expected regret $O(K m \sqrt T \log T )$ .
\end{theorem}

We do not know whether our regret upper bounds are tight; the only lower bound for this problem is an asymptotic lower bound of $\Omega ((K-m) \log(T) / \Delta')$ as $T\to\infty$, provided $\Delta'>0$, proved in~\citet[Theorem~3.1]{anantharam} for both feedback models (see~\eqref{lowbd} below for the exact form). 
There are gaps between our upper bounds and this bound and closing them is left for future work.
Further asymptotic lower bounds were claimed in~\citet[Section~3]{emily_multiplayer}, but the authors found a mistake later, see~\citet{erratum}.
\label{lower_bounds}

Another interesting avenue for future research is the setting in which
the rewards are not i.i.d.\ but are chosen by an adversary.
This problem has been studied recently by \citet*{adversarial1}
and independently by \citet*{adversarial2}.

A third possible research direction is to study this problem from a (competitive) game-theoretic point of view: each player wants to maximize her own reward and the players are not required to run the same algorithm. Can we redefine the notion of reward so the players are better off running the same algorithm?
What happens if most players are running the same, standard algorithm but there are some outliers who are selfish and deviate from the standard algorithm?
See~\citet{selfish} for recent results in this direction.

The three algorithms proving Theorem~\ref{thm:secondmain} are quite similar.
All of our algorithms have the property that,
eventually, each player fixates on one arm.
This can be viewed as reaching an equilibrium in a game-theoretic framework, where the actions correspond to the arms and the utility of each action is the mean of the arm if no two players choose that action and zero otherwise.
Games with the property that 
``if two or more players choose the same action then their reward is zero'' are called {\em anti-coordination games}.
Using our techniques for multiplayer bandits, we also provide an algorithm for converging to an approximate Nash equilibrium quickly in such a game.

More precisely, we define a \emph{stochastic anti-coordination game} as follows:
for each player $i \in [m]$ and action $j \in [K]$, there is a parameter $\mu_j^i \in [0,1]$ such that, if player $i$ performs  action $j$ while no other player performs it, she will get a random reward in $[0,1]$ with mean $\mu_j^i$,
while if two or more players perform the same action, 
all get reward 0.
An assignment of players to actions is called an {\em $\eps$-Nash equilibrium} if no player can improve her expected reward by more than $\eps$ by switching to another action while other players' actions are unchanged.
Then, we would like to design an algorithm that reaches an $\eps$-Nash equilibrium quickly.
We prove the following theorem in this direction.

\begin{theorem}
	\label{thm:anticoordination}
	There is a distributed algorithm
	that, with probability at least $1-\delta$,
	converges to an $\eps$-Nash equilibrium
	in any stochastic anti-coordination game within
	$O(\log(K/\delta) (K/\eps^2+K^2/\eps)) $ many rounds.
\end{theorem}

Note that this theorem is proved in the setting in which the players do not observe collisions;  in particular, they do not observe the actions of other players.
However, we are still making the Assumptions 1--3 (note there is no parameter $T$ in this case).
Moreover,  we assume each player also has the option of choosing a dummy action with zero reward.
This is a realistic assumption in most applications.

In the next section, we review some related work.
Theorems~\ref{thm:firstmain}
and~\ref{thm:secondmain} are proved in
Sections~\ref{sec:firstproof}
and~\ref{sec:secondproof}, respectively.
In Section~\ref{sec:relaxing} we discuss how to relax Assumptions 1--3 above.
Finally, the proof of Theorem~\ref{thm:anticoordination}
appears in Section~\ref{sec:thirdproof}.

\section{Related work.}

\subsection{Model with collision information.}
Multiplayer multi-armed bandits were introduced by~\citet*{anantharam} and further studied by~\citet*{komiyama}.
They studied a centralized setting where there is a single center that 
observes the rewards of all players and controls the players.
The distributed setting was introduced by~\citet{liuzhao}, who gave an algorithm with expected regret bounded by
$\kappa \log T$, with $\kappa$ depending on the game parameters, $m$, $K$, and the arm means.
They also showed that any algorithm must have regret $\Omega(\log T)$.
The dependence of $\kappa$ on the parameters was further improved by
\citet*{anandkumar2011distributed,musicalchair,emily_multiplayer}.

\citet{musicalchair} introduced a ``musical chairs'' subroutine to reduce the number of collisions; we have further developed and used this subroutine in our algorithms.
Their final algorithm requires the knowledge of $\Delta$ and its expected regret is bounded by
$O(m^2 + mK^2 \ln(T) + mK \log(T)/\Delta^2)$, which is at least as large as the bound of Theorem~\ref{thm:firstmain}.

Let $\log(\cdot)$ denote the natural logarithm, and define $\kl(x,y)\coloneqq x \log(x/y) + (1-x)\log((1-x)/(1-y))$.
\citet{emily_multiplayer} developed an algorithm whose regret is bounded by
\[
O\left(  \log (T) \right) \left( \sum_{i=m+1}^{K} \frac{m}{\kl(\mu_i, \mu_m)} + 
\sum_{1\leq i < j \leq K}
\frac{m^3}{\kl(\mu_j,\mu_i)} \right) ,
\]
This bound is not comparable with the bound of Theorem~\ref{thm:firstmain} in general;
however if $\mu_1=\dots=\mu_m=1/2$
and
$\mu_{m+1}=\dots=\mu_K=1/2-\Delta$,
then their bound becomes $O(m^3 K^2 \log (T)/\Delta^2)$,
which is worse than our bound by a multiplicative factor of $m^2K$.

Since the first version of this paper appeared on arXiv in August 2018,
the multiplayer bandits problem has attracted lots of attention and 
new results have been proved, which improve our bounds in some regimes.
One of the main new ideas in some of these algorithms is to use collisions as a means of communication between players.

\citet*[Theorem~1]{vianney} presented the algorithm SIC-MMAB with expected regret
\[
O\left(
\sum_{i=m+1}^{K} \min \left\{ \frac{\log T}{\mu_m - \mu_i}, \sqrt{T \log T} \right\}
+ mK \log T + m^3 K \log^2 \left( \min\left\{ T, \frac{\log T}{\Delta^2}  \right\}\right)
\right).
\]

An asymptotic regret lower bound (as $T\to\infty$) of 
\begin{equation}
\log(T) \sum_{i: \mu_i < \mu_m} \frac{\mu_m - \mu_i}
{\operatorname{kl}(\mu_i,\mu_m)}\label{lowbd}
\end{equation}
was proved in \citet[Theorem~3.1]{anantharam}.
Assuming all arm means are distinct, \citet*[Theorem~1]{improve2} presented  the algorithm DPE1 achieving this lower bound asymptotically as $T$ approaches infinity.

\subsection{Model without collision information.}
The model was introduced by~\citet{iot} and further studied by~\citet{emily_multiplayer}.
These papers introduced an algorithm and studied it empirically but gave no theoretical guarantee.

Assuming a positive lower bound $\mu_{\min}$ is known for all the arm means, 
\citet*[Theorem~3]{vianney} presented the algorithm SIC-MMAB2 whose expected regret is 
\[
O\left( 
\sum_{i=m+1}^{K} \min \left\{ \frac{m \log T}{\mu_m - \mu_i}, \sqrt{mT \log T} \right\}
+ \frac{mK^2 \log T}{\mu_{\min}}
\right).
\]
\citet*[Theorem~2]{improve1} presented the algorithm EC-SIC with expected regret bound
\[
O\left(
\sum_{i=m+1}^{K}  \frac{\log T}{\mu_m - \mu_i}
+ \log T \left( \frac{mK}{\mu_{\min}} + \frac{m^2K\log(1/\Delta)}{E(\mu_{\min})}  \right)
\right),
\]
where $E(\cdot)$ is a certain information-theoretic function called
{\em Gallager's error exponent function for the Z-channel}.

Assuming the players have access to shared randomness, \citet*[Theorem~1.1]{no_collision} gave an algorithm with regret $O(mK^{11/2} \sqrt{T \log T})$ with the additional property that, with probability $1-1/T$,
no collision occurs between players.

\subsection{Other models.}
\citet*{sharedrewards} studied a version of the problem in which if more than one players pull an arm, the reward is shared among them.

\citet{avner2014concurrent,musicalchair,dynamic,vianney} studied a dynamic version of the problem, in which the players can leave the game and new players can arrive, and proved sublinear regret bounds.

In the ``heterogeneous'' variant of the problem, the arms' reward distributions can differ across players; for results on this version, see, e.g.,~\citet{heter} and the references therein.

Finally,~\citet*{markets} studied a heterogeneous and competitive variant, where the goal is to reach a stable matching as soon as possible.

\section{Proof of Theorem~\ref{thm:firstmain}.}
\label{sec:firstproof}
\SetAlgoProcName{Subroutine}{Subroutine}
In this section, we consider only the feedback model in which the collisions are not observed and give an algorithm with regret $O ( m K \log (T) / \Delta^2) $.
The algorithm outline is simple: first, each player builds estimates for the arm means by random exploration until she detects the best $m$ arms with high probability.
Second, once the $m$ players have detected the $m$ best arms, they distribute these among themselves.

We now explain the details.
Each of the players execute the same algorithm, which has four phases, described next.
Note that the phases are not synchronized; that is, each phase may have different starting and stopping times for each player.
Let $g \coloneqq 128 K \log (3 K m^2 T^2)$.

\begin{description}
	\item{Phase 1:}
	The player pulls arms uniformly at random and maintains an estimate for the mean of each arm---the estimate for arm $i$ is the average reward received from arm $i$ divided by $(1-1/K)^{m-1}$. 
	Note that, provided other players are also pulling arms uniformly at random, $(1-1/K)^{m-1}$ is precisely the probability of {\em not} getting a conflict for a random pull, hence the player indeed has an unbiased estimate for $\mu_i$.
	In other words, for any round $t$ that arm $i$ is pulled and reward $r(t)$ is received, since collisions and rewards are independent, we have (recall~\eqref{rewardeq})
	$$\mu_i = \E Y_{i,t} = 
	\frac{\E r(t)}{ \E(1 - C_{i}(t))}
	= 
	\frac{\E r(t)}{ (1-1/K)^{m-1}}.$$ 
	For each round $t$,
	the player maintains a sorted list $\muhat_{i_1,t}\geq \dots \geq \muhat_{i_K,t}$ of estimated means. Let $\tau$ be the first round when
	$\muhat_{i_m,\tau} - \muhat_{i_{m+1},\tau}\geq 3\sqrt{g/\tau}$.
	The first phase finishes at the end of round $\tau$.
	We will prove that by this time, the player has learned the best $m$ arms with high probability, and so she has a list $G\subseteq [K]$ of $m$            arms with the highest  means.
	
	\item{Phase 2:}
	For $24\tau$ rounds, the player just pulls arms uniformly at random.
	
	\item{Phase 3:}
	The player runs a so-called \emph{musical chairs algorithm} until it occupies an arm.
	In each round, she pulls a uniformly random arm $i \in G$;
	if she gets a positive reward (which means no other player has pulled arm $i$), we say the player has ``occupied'' arm $i$, and this phase is finished for the player.
	Note that, by construction, at most one player will occupy any given arm.
	
	\item{Phase 4:}
	The player pulls the occupied arm forever.
\end{description}

The pseudocode is shown in Algorithm~\ref{alg1}.
We next analyze the regret of this algorithm, starting with some preliminary lemmas.

\SetInd{0.5em}{1em}
\SetNlSkip{1em}
\begin{algorithm}
	\caption{the algorithm for Theorem~\ref{thm:firstmain}}\label{alg1}
	\DontPrintSemicolon
	\KwIn{number of players $m$, 
		number of arms $K$, 
		number of rounds $T$}
	$g \longleftarrow 128 K \log (3 K m^2 T^2)$\;
	$\muhat_i \longleftarrow 0$ for all $i\in[K]$\;
	\tcp{Phase 1 }
	$\tau \longleftarrow 0$\;
	\Repeat
	{$\muhat_{i_m} - 
		\muhat_{i_{m+1}}\geq 3\sqrt{g/\tau}$}
	{pull a uniformly random arm $i$\;
		$\muhat_i \longleftarrow$
		average reward  from arm $i$ divided by $(1-1/K)^{m-1}$\;
		Sort the $\mathbf {\muhat}$ vector as 
		$\muhat_{i_1}\geq \dots \geq \muhat_{i_K}$\;
		$\tau \longleftarrow \tau+1$\;
	}
	$\textnormal{Best-$m$-arms}\longleftarrow \{i_1, i_2, \dots, i_m\}$\;
	\tcp{Phase 2 }
	\lFor{$24\tau$ rounds}{pull arms uniformly at random}
	\tcp{Phase 3 }
	$i \longleftarrow $ MusicalChairs1 (Best-$m$-arms)\;
	\tcp{Phase 4 }
	Pull arm $i$ until end of game
\end{algorithm}

\begin{procedure}
	\DontPrintSemicolon
	\KwIn{set $A\subseteq[K]$ of target arms}
	\KwOut{index of an occupied arm}
	\While{true}{
		\quad pull an arm $i\in A$ uniformly at random\;
		\quad  \lIf(\tcp*[h]{arm $i$ is occupied}){positive reward is received}
		{output $i$ }
	}
	\caption{MusicalChairs1($A$)}
\end{procedure}

We will use the following versions of Chernoff-Hoeffding concentration inequalities.
\begin{proposition}[\protect{\cite[Theorem~2.3]{concentration_mcdiarmid}}]
	\label{prop:concentration}
	Let the random variables $X_1,\dots,X_n$ be independent, with $0\leq X_k \leq b$ for each $k$ and some fixed $b$.
	Let $\widehat{\mu} = \sum X_k/n$ and $\mu = \E \muhat$. Then we have,
	
	\noindent(a) for any $t\geq 0$, 
	\[ \p{|\widehat{\mu} - \mu| > t } < 2 \exp (-2 n t^2/b^2),\]
	(b) and if $b=1$, then for any $\eps>0$,
	\[ \p{\widehat{\mu} < (1-\eps) \mu } < \exp (-\eps^2 n \mu / 2).\]
\end{proposition}

We start the analysis with Lemma~\ref{lem:goodestimateofmeans}, proving that the  mean estimates are close enough to the true means with high probability.
Then, in Lemma~\ref{cor:fitness} we prove that with high probability the two first phases will not take too long and once they are finished, all players have learned the best $m$ arms.
Finally, Lemma~\ref{lem:mc1} analyzes the MusicalChairs1 subroutine and shows that with high probability it does not take too long for each player to occupy a distinct good arm.

\begin{lemma}\label{lem:goodestimateofmeans}
	Consider any fixed player and let
	$\muhat_{i,t}$ denote her estimated mean for arm $i$ after $t$ rounds of Phase 1.
	Then we have
	\[
	\p{\exists i\in[K], t\in[T] : 
		|\muhat_{i,t}-\mu_i| > \sqrt {g/t}}
	< 3 KT \exp (-g/128K).
	\]
\end{lemma}
\begin{proof}
	Fix an arm $i\in[K]$.
	Observe that 
	$0\leq \muhat_{i,t} \leq 1 / (1-1/K)^{m-1} 
	< 1/(1-1/K)^{K} \leq 4$, so we have $|\muhat_{i,t}-\mu_i|\leq4$ deterministically, so for $t\leq g/16$ we have
	$|\muhat_{i,t}-\mu_i| \leq \sqrt {g/t}$.

	Next, fix some $t>g/16$.
	Let $T_i(t)$ denote the number of times this player has pulled arm $i$ by round $t$, which is a binomial 
	random variable with mean $t/K$, hence Proposition~\ref{prop:concentration}(b) implies
	$\p{T_i(t) < t/(2K)}
	< \exp(-t/8K)$.
	Thus, the union bound gives
	\[
	\p{|\muhat_{i,t}-\mu_i| > \sqrt {g/t}}
	< 
	\exp(-t/8K) + 
	\max_{\frac{t}{2K} \leq s \leq t}
	\p{|\muhat_{i,t}-\mu_i| > \sqrt {g/t} \: \Big| \: T_i(t)=s}.
	\]
	Also, conditioned on any $s\in[t]$,
	$|\muhat_{i,t}-\mu_i|$ is the difference between an empirical average of $s$ i.i.d.\ random variables bounded in $[0,4]$ and their expected value, thus Proposition~\ref{prop:concentration}(a) gives\label{hoeffding_used}
	\[
	\p{|\muhat_{i,t}-\mu_i| > \sqrt {g/t} \: \Big| \: T_i(t)=s} < 2\exp(-s g / 8 t),
	\]
	giving
	\[
	\p{|\muhat_{i,t}-\mu_i| > \sqrt {g/t}}
	< 
	\exp(-t/8K) + 
	2\exp(-g / 16K) < \exp (-g/128K) + 2\exp(-g/16K),
	\]
	since $t > g/16$. A union bound over $t\in[T]$ and the $K$ arms concludes the proof.
	\end{proof}

In the next lemma, we prove that with high probability the first two phases will not take too long, and once they are finished, all players have learned the best $m$ arms.

\begin{lemma}\label{cor:fitness}
	With probability at least $1-1/mT$, the following are true.
	
	(i) All players have learned the best $m$ arms by end of their Phase 1.
	
	(ii) We have 
	\[g / \Delta^2 \leq \tau \leq 25 g /  \Delta^2\]
	for all players (where we recall that $\tau$ is the round at which phase 1 finishes).
	
	(iii) The first two phases are finished for all players after at most $625 g / \Delta^2$ many rounds.
\end{lemma}
\begin{proof}
	By the choice of $g = 128 K \log (3 K m^2 T^2)$, Lemma~\ref{lem:goodestimateofmeans} and a union bound over the $m$ players, with probability at least $1-1/mT$, all players' mean estimates are off by an additive error of at most $\sqrt{g/t}$, uniformly for all arms and all $t\in[T]$. We next explain how the three parts of the lemma follow from this.
	
	Part (i) follows 
	by noting that a player would stop Phase 1 when she has found a gap of size $3\sqrt{g/\tau}$ between the $m$th the and the $(m+1)$th arm.
	By this time, she has learned the means of all arms within an additive error of $\sqrt{g/\tau}$, therefore by the triangle inequality, she has correctly determined that the $m$th mean is larger than the $(m+1)$th mean, whence she has learned the best $m$ arms.
	
	For part (ii), using the triangle inequality and the definition of $\tau$, we have
	\begin{equation*}
	3 \sqrt \frac{g}{ \tau}
	\leq \muhat_{m,\tau}-\muhat_{m+1,\tau}
	=
	(\muhat_{m,\tau}-
	\mu_m)
	+
	(\mu_m-\mu_{m+1})
	+
	(\mu_{m+1}-
	\muhat_{m+1,\tau})
	\leq \sqrt\frac{g}{ \tau}
	+ \Delta + \sqrt\frac{g}{ \tau},
	\end{equation*}
	whence $\tau \geq g/\Delta^2$.
	On the other hand,
	by time $t = 25 g /  \Delta^2$,
	\begin{equation*}
	\muhat_{m,t}-\muhat_{m+1,t}
	\geq
	(\mu_m-\mu_{m+1})
	- |\muhat_{m,t}-
	\mu_m|
	-|\mu_{m+1}-
	\muhat_{m+1,t}|
	\geq
	\Delta - 2 \sqrt{\frac g t} = 3 \sqrt{\frac g   t},
	\end{equation*}
	whence $\tau \leq t$.
	
	Part (iii) follows from part (ii) by noting that the duration of Phase 2 is $24\tau$ rounds.\end{proof}

Curious readers may wonder about the role of Phase 2 and ask, ``Why cannot a player proceed to Phase 3 right after she has learned the best $m$ arms?''
The answer is that Phase 2 is designed to help other players to find the best $m$ arms.
Indeed, it is possible that a player finishes Phase 1 by round $g/\Delta^2$, but the algorithm asks her to continue pulling arms at random so other players continue to have unbiased estimators for the means for at least $24g/\Delta^2$ more rounds, at which point we are guaranteed that all players have finished their Phase 1.
Otherwise, if a player switches to Phase 3 too quickly, then this would skew the collision probabilities and other players will not have unbiased mean estimates.

We now proceed to analyzing Phase 3, the musical chairs subroutine.
By this point, all players have learned the best $m$ arms, hence they just want to share these $m$ arms among themselves as quickly as possible.
The next lemma shows that this will not take too long.
Note that by definition of the subroutine,
once this phase is finished, each player has occupied a distinct arm.

\begin{lemma}
	\label{lem:mc1}
	With probability at least $1-1/mT$, Phase 3 takes at most $4 m \log (m^2 T)/\Delta$ many rounds for all players.
\end{lemma}
\begin{proof}
	Since each reward $Y_{i,t}$ takes value in $[0,1]$, we have
	$\p{Y_{i,t}>0}\geq \E Y_{i,t}$.
	Fix any player in her Phase 3 who has not occupied an arm,
	and suppose there are still $a$ unoccupied arms available for her. (There are $m$ players, and each occupies at most one arm, hence $a\geq1$.)
	Whenever she tries to occupy an unoccupied arm, her success probability is at least 
	\[
	\frac{a}{m} \Delta (1-1/m)^{m-a}
	\geq \Delta/4m.
	\]
	Here, $\frac{a}{m}\geq 1/m$
	is the probability that she pulls an unoccupied arm,
	$\Delta \leq \mu_m$ is a lower bound on the probability that that arm produces a positive reward,
	and
	$(1-1/m)^{m-a}\geq1/4$ is the probability that no other player pulls that arm.
	Hence, the probability that the player has not occupied an arm after $t$ attempts can be bounded by $(1-\Delta/4m)^t\leq \exp(-t\Delta/4m)$.
	Letting $t=4 m \log (m^2 T)/\Delta$ makes this probability $\leq 1/Tm^2$.
	Applying the union bound over all players completes the proof.
\end{proof}

\begin{proof}[Proof of Theorem~\ref{thm:firstmain}]
	By Lemma~\ref{cor:fitness}
	and Lemma~\ref{lem:mc1},
	with probability at least $1-2/mT$,
	the first three phases finish for all players after at most
	\(
	625 g / \Delta^2 + 4 m \log (m^2 T) / \Delta=O ( K \log (KT) / \Delta^2)
	\)
	many rounds.
	After this time, each player has occupied one of the best $m$ arms and different players have occupied distinct arms.
	During each round, the regret is at most $m$, hence
	the total regret incurred during the first three phases is bounded by
	$O ( mK \log (KT) / \Delta^2)$ and the regret afterwards would be 0.
	On the other hand, with the remaining $2/mT$ probability, the regret is at most $mT$.
	Therefore, the expected regret is at most
	\(
	O ( m K \log (KT) / \Delta^2) + 2,
	\)
	as required.
	The $O(\log (KT))$ can be replaced with $O(\log (T))$, noting that
	\[
	\min \{mT, O ( m K \log (KT) / \Delta^2) \}
	= O ( m K \log (T) / \Delta^2).
	\]	\end{proof}
	
\section{Proof of Theorem~\ref{thm:secondmain}.}
\label{sec:secondproof}
Recall that Theorem~\ref{thm:secondmain} has three parts focusing on three different settings:
in part (a), we do not observe the collision information, but we know a lower bound for $\mu_m$;
in part (b), we observe the collision information,
while in part (c), we do not observe the collision information, but we allow the players to leave the game at points of their choice.
We start by proving part (a), and then we explain how the algorithm and the analysis can be modified to prove parts (b) and (c).

\subsection{Algorithm for Theorem~\ref{thm:secondmain}(a).}
We describe the algorithm each player executes, first informally and then formally.
Recall that $\mu$ is a lower bound for $\mu_m$ that all players know in advance. 
The algorithm has a parameter $\nu$ which we  set it equal to $\mu$ for this part.
We say an arm is \emph{$\nu$-good} if its mean is at least $\nu$;
otherwise, we say it is \emph{$\nu$-bad}.

The player maintains estimates $\muhat_1,\dots,\muhat_K$ for the means,
which approach the actual means as the algorithm proceeds.
She also keeps a {\em confidence interval} for each arm $j$, which is centered at $\muhat_j$ and has the property that $\mu_j$ lies in this interval with sufficiently high probability.
If arm $j$ has been pulled $s$ times, this interval has length $O(\sqrt{\log(T)/s})$.
Once the player makes sure that some arm is not among the best $m$ arms, she marks it as ``bad'' and puts it in a set $B$.
This can happen if it is determined that the arm is $\nu$-bad
or
that there are at least $m$ arms whose confidence intervals lie strictly above this arm's interval
(we say interval $[c,d]$ lies strictly above $[a,b]$ if $b<c$).
On the other hand, once the player makes sure that some arm is within the best $m$ arms, she marks it as a ``golden'' arm and puts it in a set $G$.
This would happen as soon as there are at least $K-m$ arms that are determined to be bad or whose confidence intervals lie strictly below this arm.
Other arms, whose status is yet unknown, are called ``silver'' arms and kept in a set $S$.

Initially, all arms are silver.
The algorithm proceeds in epochs with increasing lengths.
In each epoch, the player explores all silver arms and hopes to characterize each silver arm as golden or bad at the end of the epoch.
As time proceeds, arms whose means are far away from the $m$th arm will be characterized as either golden or bad.
Golden arms will be occupied quickly, and bad arms will not be pulled again---this will control the algorithm's regret.

Special care is needed to ensure all players explore all silver arms without conflicts; this is done via careful executions of a suitable musical chairs subroutine, called MusicalChairs2, explained in the next paragraph.
In each epoch, each player maintains a set $E$ of explored arms, which is empty when the epoch starts.
The epoch has $K+m-1$ iterations.
In each iteration, if there exists some arm in $S \setminus E$ (i.e., an unexplored silver arm), the player tries to occupy such an arm;
otherwise, the player has finished exploring the arms in $S$, and so she will try to occupy and pull an arbitrary arm from $S$, while other players are exploring their silver arms.
Note that by the assumption that $\mu_m\geq \mu$ and $\nu=\mu$, any $\nu$-bad arm is bad.
The length of the MusicalChairs2 subroutines are chosen such that each $\nu$-good arm in $S$ that is not marked as golden by any other player will be explored in each epoch by each player.
Thus, if an arm is not explored by the end of an epoch,  either the arm is $\nu$-bad or the arm is golden and is occupied by another player in the beginning of the epoch. 
The two cases will be distinguished by checking the empirical reward received from that arm.

We now describe the MusicalChairs2 subroutine, which is different from MusicalChairs1 from the previous section because different players may have different ``target sets'' now.
(A target set is a subset of the arms that a player wants to explore.)
For any target set $A$ of arms and any number $\alpha$ of rounds, this subroutine consists of precisely $\alpha$ rounds as follows:
in each round, the player pulls a uniformly random arm $j \in [K]$.
If she gets a positive reward
and $j\in A$, then
she occupies arm $j$, pulls arm $j$ for the remaining rounds, and outputs $j$.
Otherwise, she tries again.
If after $\alpha$ rounds she has not occupied any arm, she outputs the dummy index 0.
The pseudocode for MusicalChairs2 appears below.

\SetKw{and}{and}
\begin{procedure}
	\DontPrintSemicolon
	\KwIn{set $A\subseteq[K]$ of target arms, number $\alpha \in \{1,\dots\}$ of rounds}
	\KwOut{if an arm is occupied, the output is its index; otherwise, the output is 0}
	\For{$i\longleftarrow1$ \KwTo $\alpha$}
	{
		pull an arm $j\in [K]$ uniformly at random\;
		\If(\tcp*[h]{arm $j$ is occupied})
		{$j\in A$ \and positive reward is received}
		{
			pull arm $j$ for the remaining $\alpha-i$ rounds\;
			output $j$\; 
		}
	}
	\tcp*[h]{no arm is occupied}\;
	output 0\;
	\caption{MusicalChairs2($A$, $\alpha$)}
\end{procedure}

The pseudocode for Theorem~\ref{thm:secondmain}(a) appears in Algorithm~\ref{alg2} below.
Note that this algorithm is synchronized---for all players, the epochs and the iterations within the epochs begin and end at the same round.

\begin{algorithm}
	\small
	\caption{the  algorithm for Theorem~\ref{thm:secondmain}(a)}\label{alg2}
	\SetAlgoLined
	\DontPrintSemicolon
	\KwIn{number of players $m$, 
		number of arms $K$, 
		number of rounds $T$,
		lower bound $\mu$ for $\mu_n$}
	$\nu\longleftarrow\mu$, $g \longleftarrow \log(4m^3T^2K)/2$, $\alpha \longleftarrow 4 K \log (6 Km^2 T)/\nu$\;
	$\muhat_i \longleftarrow 0$ for all $i\in[K]$\;
	$G\longleftarrow \emptyset, B\longleftarrow \emptyset, S\longleftarrow[K]$\;
	\For{epoch $i=1,2,\dots,$}
	{
		$j \longleftarrow $MusicalChairs2($G, \alpha$)\nllabel{occupygolden}
		\tcp*[h]{occupy a golden arm if possible}\;
		\lIf
		{$j >0$} {disregard the rest of the algorithm and pull arm $j$ forever}
		$E \longleftarrow \emptyset$ \nllabel{setE}\tcp*[h]{the set of arms that have been explored  in this epoch}\;
		\For(\tcp*[h]{goal: explore all silver arms by end of this loop}){$K+m-1$ iterations}
		{
			$j\longleftarrow 0$\;
			\eIf{$E\neq S$}
			{$j \longleftarrow$ MusicalChairs2 $(S \setminus E, \alpha)$ \tcp*[h]{occupy an unexplored arm}\nllabel{trytooccupy}}
			{randomly pull arms in the next $\alpha$ rounds \\ \tcp*[h]{all silver arms have been explored in this epoch; waste time for $\alpha$ rounds}}
			\eIf{$j=0$}{$j \longleftarrow$ MusicalChairs2 $(S, \alpha)$ \tcp*[h]{occupy any silver arm}}{pull arm $j$ for $\alpha$ rounds \tcp*[h]{keep on occupying arm $j$}}
			pull arm $j$ for $2^i$ rounds and let $\muhat_j\longleftarrow$ the average reward received \nllabel{estimationLine}\;
			update confidence interval of arm $j$ to 
			$\left[\muhat_j - \sqrt{g/2^i}, \muhat_j + \sqrt{g/2^i}\right]$ \nllabel{updateLine}\;
			insert $j$ into $E$ \tcp*[h]{add $j$ to the set of explored arms}\;
		}
		\ForEach(\tcp*[h]{put the unexplored arms in either $G$ or $B$}){$j\in S\setminus E$\nllabel{unexploredarms}}
		{
			\eIf{$\muhat_j - \sqrt{g/2^{i-1}} > \nu$\nllabel{verypicky}}
			{move $j$ from $S$ to $G$ \tcp*[h]{arm $j$ is golden but occupied by another player}\nllabel{jing}}
			{move $j$ from $S$ to $B$ \nllabel{jinb}\tcp*[h]{arm $j$ has mean $< \nu$}}
		}
		\ForEach(\tcp*[h]{update the golden and bad arms}){$j\in S$\nllabel{updategolden}}
		{
			\uIf{there exist at least $m-|G|$ arms $\ell\in S$ with
				$\muhat_{\ell} -\sqrt{g/2^i} > \muhat_j+\sqrt{g/2^i}$\nllabel{badconditions}}
			{move $j$ from $S$ to $B$}
			\uElseIf{
				$\muhat_j > \nu + 3 \sqrt{g/2^i}$ \and
				there exist at least $K-m-|B|$ arms $\ell\in S$ with
				$\muhat_{\ell} + \sqrt{g/2^i} < \muhat_j-\sqrt{g/2^i}$\nllabel{goldenconditions}}
			{move $j$ from $S$ to $G$\nllabel{updategoldenend}}
		}
	}
\end{algorithm}

To analyze Algorithm~\ref{alg2},
we define two bad events: 
failure of some MusicalChairs2 subroutine (handled by Corollary~\ref{cor:mcguarantee} below)
or incorrectness of some confidence interval (handled by Lemma~\ref{lem:bad_event} below).
After proving their unlikeliness, we will bound the regret assuming no bad events happen.

\subsection{Bounding the failure probability of MusicalChairs2.}
We next prove a  lemma bounding the failure probability of this subroutine, but first we formally define the notion of success.

\begin{definition}[$\nu$-successful MusicalChairs2 subroutine]
	Let $\nu\in[0,1]$ be arbitrary.
	Suppose that a subset of players are executing the MusicalChairs2 subroutine simultaneously for some consecutive rounds (call these the {\em participating} players), while any other player either pulls uniformly random arms or pulls a fixed arm during these rounds.\label{def:successful}
	The participating players may have different target sets.
	We say a participating player $p$ with target set $A_p$ is {\em $\nu$-successful} if, by the end of the subroutine,
	either she occupies an arm in $A_p$ 
	or all $\nu$-good arms in $A_p$ are occupied by someone else (participating or otherwise).
	A player is {\em $\nu$-failed} if she is not $\nu$-successful.
	Moreover, we say the subroutine is {\em $\nu$-successful} if all participating players are $\nu$-successful, and we say the subroutine has {\em $\nu$-failed} if at least one participating player has $\nu$-failed.
\end{definition}

\begin{lemma}\label{lem:mc2guarantee}
	Let $\nu\in[0,1]$ and let $\alpha$ be a positive integer.
	For MusicalChairs2 of length $\alpha$, the $\nu$-failure probability of any fixed player is upper bounded by
	$\exp(-\alpha\nu/4K)$ if $m\leq K$ and by
	$\exp(-\alpha\nu\exp(-2m/K)/K)$ in general.
\end{lemma}

\begin{proof}
	Fix a player with target set $A$.
	At any round during the subroutine, suppose the player has not occupied an arm and that there are still $a\geq1$ $\nu$-good unoccupied arms  in $A$. 
	Whenever she tries to occupy one of her target arms, her success probability is at least 
	\[
	\frac{a}{K} \nu (1-1/K)^{m-1}
	\geq \nu\exp(-2m/K)/K.
	\]
	Here, $\frac{a}{K}\geq 1/K$
	is the probability that she pulls a $\nu$-good unoccupied arm in her target set,
	$\nu$ is a lower bound on the probability that that arm produces a positive reward,
	and
	$(1-1/K)^{m-1} \geq \exp(-2m/K)$ is the probability that no other player pulls the same arm.
	(Note that her success probability may indeed be larger because she may also occupy $\nu$-bad arms in her target set.)
	Hence, the probability that she has not occupied an arm after $\alpha$ attempts can be bounded by $(1-\nu\exp(-2m/K)/K)^\alpha \leq \exp(-\alpha\nu\exp(-2m/K)/K)$.
	If $m\leq K$, the argument is identical, but we use the tighter bound
	$(1-1/K)^{m-1}> (1-1/K)^K \geq 1/4$.\end{proof}

Applying a union bound over the $m$ players gives the following corollary.

\begin{corollary}\label{cor:mcguarantee}
	Let $\nu\in[0,1]$ and let $\alpha$ be a positive integer.
	The $\nu$-failure probability of a MusicalChairs2 subroutine of length $\alpha$ is not more than $m\exp(-\alpha\nu/4K)$
	if $m\leq K$ and 
	$m\exp(-\alpha\nu\exp(-2m/K)/K)$ in general.
\end{corollary}

\subsection{Proof of Theorem~\ref{thm:secondmain}(a).}
\label{sec:wholealgorithm}
As explained in later subsections, the proofs of Theorems~\ref{thm:secondmain} (b, c) differ only in values of the parameters $\nu, \alpha, g$.
For this part, we put $\nu=\mu$ and
define
\begin{equation}
\alpha=4 K \log (6 Km^2 T)/\nu
\textnormal{ and }
g=\log(4m^3T^2K)/2.
\label{eqag}
\end{equation}

We first define the two bad events formally.
The first bad event is that some MusicalChairs2 subroutines $\nu$-fail,
and the second bad event is that some player's confidence interval is incorrect, i.e., the actual mean does not lie in the confidence interval.
We start by bounding the probability of the bad events.

\begin{lemma}\label{lem:bad_event}
	Let $\nu\in[0,1]$ be arbitrary and define  $\alpha,g$ as in~\eqref{eqag}.
	The probability that some bad event happens is at most $1/mT$.
\end{lemma}
\begin{proof}
	The probability that some MusicalChairs2 subroutine $\nu$-fails is bounded by
	$m \exp(-\alpha \nu / 4 K)$ by Corollary~\ref{cor:mcguarantee}.
	By a union bound over the (at most $T$) epochs and the $1+2(K+m-1)\leq 3Km$ times MusicalChairs2 is executed in each epoch, the probability that some MusicalChairs2 subroutine $\nu$-fails is at most
	\(3 K m \times T \times m \exp(-\alpha \nu / 4 K) \leq 1/2mT\), as $\alpha=4 K \log (6 Km^2 T)/\nu$.
	
	Whenever a confidence interval for some arm $j$ is updated in some epoch $i$ (Line~\ref{updateLine}), the arm has been pulled precisely $2^i$ times right before that (Line~\ref{estimationLine}).
	Hence, the probability that some confidence interval is incorrect for some player, say in epoch $i$, is bounded via Proposition~\ref{prop:concentration}(a) by
	\[
	\p{|\muhat_j-\mu_j| > \sqrt{g/2^i}} < 2\exp(-2 \times 2^i g / 2^i) = 2 \exp(-2g).
	\]
	By a union bound over the $m$ players, the (at most $T$)
	epochs, and the $K+m-1 \leq Km$ many updates of the confidence intervals within each epoch, the probability of some incorrect confidence interval is at most 
	\(m \times T \times Km \times 2 \exp(-2g) \leq 1/ 2mT\), as $g=\log(4m^3T^2K)/2$, completing the proof.\end{proof}

We are now ready to prove Theorem~\ref{thm:secondmain}(a).

\begin{proof}[Proof of Theorem~\ref{thm:secondmain}(a)]
	We bound the regret assuming no bad event happens, and the bound for the expected regret follows as in the proof of Theorem~\ref{thm:firstmain}. We first prove four\label{sketch} deterministic claims and then bound the regret. Informally, these claims are:
	\begin{enumerate}
		\item Any silver arm is explored at least $2^i$ times by each \emph{active player} during epoch $i$. (An {active player} is one that has not occupied a golden arm yet.)
		\item No player makes a mistake in marking an arm as golden or bad.
		\item Any arm whose mean is much smaller than $\mu_m$ will be marked by all players as bad quickly.
		\item Any arm whose mean is much larger than $\mu_m$ will be marked by all players as golden quickly and occupied by one of them quickly.
	\end{enumerate}
	
	We now proceed to the formal argument.
	Note that each epoch has two types of rounds:
	\emph{estimation rounds} (Line~\ref{estimationLine}), in which each arm is pulled by at most one player, during which she updates her estimate of its mean,
	and other rounds, during which some players are executing MusicalChairs2, hence we call them \emph{MusicalChairs2 rounds}.
	
	Observe that, since there are at least $m$ many $\nu$-good arms (here we are using the fact $\nu\leq\mu$), we always have $|G\cup S| \geq m$, and since the MusicalChairs2 subroutines are always $\nu$-successful, there will be no collision during the estimation rounds.
	
	The first claim is the following:
	consider a player  that has just executed her Line~\ref{setE} in epoch $i$.
	Consider also a $\nu$-good arm $j$ that is silver, and
	suppose this arm is not occupied by another player as a golden arm in their Line~\ref{occupygolden}.
	Then the claim is that the player will pull arm $j$ at least $2^i$ times during epoch $i$ 
	and 
	will put it in $E$ at the end of the $K+m-1$ iterations.
	To prove this, note that the epoch has $K+m-1$ iterations.
	In each iteration, if the player has any unexplored silver arm, in the first $\alpha$ rounds she attempts to occupy one of those (Line~\ref{trytooccupy}) while other players pull random arms.
	By Lemma~\ref{lem:deletingmaxmatchings} below and since the MusicalChairs2 subroutines are $\nu$-successful,
	after the $K+m-1$ iterations,
	each player has explored any such arm $j$. 
	Therefore, the confidence interval of each such arm will have length $2\sqrt{g/2^i}$.
	
	The second claim is that  the algorithm never makes a mistake in characterizing the arms as golden and bad.
	First, the characterizations based on confidence intervals (Lines~\ref{updategolden}--\ref{updategoldenend}) are correct because all confidence intervals are correct.
	Now fix an epoch $i$ and an arm $j$, and note that if $j\in S\setminus E$
	on Line~\ref{unexploredarms},
	that means $j$ is not explored, and that can be for one of two reasons:
	it may be a golden arm occupied by another player on her Line~\ref{occupygolden}
	or its mean may be smaller than $\nu$.
	
	Case 1: arm $j$ is a golden arm occupied by another player.
	Let $\muhat'_j$ be the average reward received from this arm by the other player.
	Suppose the arm was marked as golden by the other player in epoch $i' \leq i-1$.
	Then we must have had $\muhat'_j > \nu + 3 \sqrt{g/2^{i'}}$
	(see Line~\ref{goldenconditions}).
	This implies
	$$\mu_j \geq \muhat'_j - \sqrt{g/2^{i'}} 
	> \nu + 2 \sqrt{g/2^{i'}}
	\geq \nu + 2 \sqrt{g/2^{i-1}}.
	$$
	On the other hand, at the end of epoch $i-1$, since $j$ was silver
	and the confidence intervals were correct,
	we have
	$\muhat_j \geq \mu_j - \sqrt{g/2^{i-1}} > \nu +  \sqrt{g/2^{i-1}}$,
	hence in epoch $i$, Line~\ref{jing} is executed and the algorithm correctly marks $j$ as golden.
	
	Case 2: the mean of arm $j$ is smaller than $\nu$. Because the confidence intervals were correct at the end of epoch $i-1$, $\nu$ lies in the confidence interval for arm $j$, which has length $\sqrt{g/2^{i-1}}$.
	This means $\muhat_j - \sqrt{g/2^{i-1}} \leq \nu$, so in epoch $i$,
	Line~\ref{jinb} is executed and the player correctly marks $j$ as bad.
	
	The third claim is that any arm with mean $< \mu_m - 4 \sqrt{g/2^i}$ will be marked as bad by all players by the end of epoch $i$.
	Let $j$ be such an arm and suppose we are at the end of epoch $i$.
	By Line~\ref{badconditions} of the algorithm, it suffices to show that there exists at least $m$ arms $\ell$ such that either $\ell \in G$ or $\muhat_\ell - \muhat_j > 2 \sqrt{g/2^i}$ or both.
	In fact, this holds for all $\ell \in [m]$, since for any $\ell\in[m]$,
	if $\ell \notin G$, then $\ell \in S$, which implies
$$\muhat_\ell - \muhat_j \geq 
	\mu_\ell - \mu_j
	- |\mu_\ell- \muhat_\ell |
	- |\muhat_j - \mu_j|
	> 4 \sqrt{g/2^i} - \sqrt{g/2^i} - \sqrt{g/2^i}
	= 2 \sqrt{g/2^i}.$$
	
	The fourth claim, whose proof is similar to the third claim, is that any arm $j$ with 
	$\mu_j> \mu_{m} + 4 \sqrt{g/2^i}$ will be marked as golden by all players by the end of epoch $i$.
	The only difference is the additional condition
	$\muhat_j > \nu + 3 \sqrt{g/2^i}$, which is satisfied by any such arm. Indeed, we have
	$$\muhat_j \geq \mu_j - \sqrt{g/2^i} > \mu_m + 3\sqrt{g/2^i} \geq \nu + 3\sqrt{g/2^i}$$ 
	by correctness of confidence intervals and since $\nu\leq\mu_m$.
	
	Now, we bound the algorithm's regret.
	First, the number of epochs is fewer than $\log_2(T) < 2\log(T)$.
	The number of iterations per epoch is $K+m-1< 2K$, whence the total number of MusicalChairs2 rounds can be bounded by
	$2 \log(T) (\alpha + 4 K \alpha) \leq 10 K \alpha \log (T)$.
	We next bound the regret of the estimation rounds.
	The regret of the first epoch can be bounded by $m\cdot  (K+m-1)\cdot 2^1 < 4Km$.\label{firstepochregret}
	Next note that any arm with mean 
	$> \mu_{m} + 4 \sqrt{g/2^{i-1}}$ has been put in $G$ by the end of epoch $i-1$ by all players by the fourth claim, 
	and so some player occupies  it in the beginning of epoch $i$.
	During epoch $i$, each active player pulls either a silver
	or a golden arm, which are at most $8 \sqrt{g/2^{i-1}}$ away from the
	best available arms by the third and fourth claims.
	Since the probability that some bad event happens is
	$1/mT$ (Lemma~\ref{lem:bad_event}), and in this case the total regret can be bounded by
	$mT$, the total expected regret can be bounded by
	\begin{align}
	\overbrace{mT \times (1/mT)}^{\text{\small bad events}}
	+
	\overbrace{4Km}^{\text{\small first epoch}}
	+
	\overbrace{10 mK \alpha \log (T)}^{\text{\small MusicalChairs2 rounds}}
	&+
	\overbrace{m \times \sum_{i=2}^{\lceil \log_2(T) \rceil}
	(2K \times 2^i \times 
	8 \sqrt{g/2^{i-1}} )}^{\text{\small estimation rounds}}\notag
	\\ &= O(mK \alpha \log (T) + Km \sqrt{T \log(KT)}).\label{bound1}
	\end{align}
	Recall that $\Delta'=\min \{\mu_m - \mu_i: \mu_i < \mu_m\}$.
	Let $j$ be the smallest integer that $4 \sqrt{g/2^{j}} < \Delta'$.
	So, after the first $j$ epochs, any silver arm will have 
	mean precisely $\mu_m$, and the regret will be zero afterwards.
	Hence, the total expected regret is alternatively bounded by
	\begin{align}
	mT \times (1/mT)+
	10 K m \alpha & \log T
	+
	4Km
	+
	\sum_{i=2}^{j}
	8 \sqrt{g/2^{i-1}} (2Km) 2^i\notag
	\\& = O(K m \alpha \log (T)  + Km \log(KT) / \Delta')\label{bound2}.
	\end{align}
	Combining~\eqref{bound1} and~\eqref{bound2} gives that the expected regret is upper bounded by
	\begin{equation}
	O(K m \alpha \log (T)  + Km \min\{  \sqrt{T \log(KT)},\log(KT) / \Delta'\} ).\label{finalregretbound}
	\end{equation}
	This bound holds for $\alpha=4 K \log (6 Km^2 T)/\nu$ and any $0\leq\nu\leq\mu_m$.
	Recalling that $\nu=\mu$ 
	gives Theorem~\ref{thm:secondmain}(a).\end{proof}

The following lemma is the last piece in completing the proof of Theorem~\ref{thm:secondmain}(a).

\begin{lemma}\label{lem:deletingmaxmatchings}
	Fix an epoch and suppose that all MusicalChairs2 subroutines of Line~\ref{trytooccupy} are $\nu$-successful. Then, during the $K+m-1$ iterations of the epoch,
	each player will occupy each $\nu$-good silver arm at least once.
\end{lemma}

\begin{proof}
	Consider a bipartite graph with one part being the $m$ players and the other part the $K$ arms, with an edge between a player and an arm if the arm is silver and unexplored for that player.
	Say an edge is \emph{good} if the corresponding arm is $\nu$-good.
	Say two edges are \emph{neighbors} if they share a vertex, and the \emph{degree} of an edge is its number of neighbors.
	Initially, the degree of each edge is at most $K+m-2$.
	Whenever the MusicalChairs2 subroutine in Line~\ref{trytooccupy} is executed, the set of edges corresponding to players and their occupied arms forms an edge-matching in this graph,
	i.e., a set of edges such that no two of them are neighbors.
	Moreover, since the MusicalChairs2 subroutine is $\nu$-successful by assumption, this matching $M$ has the property that, for any good edge $e$, either $e\in M$ or some neighbor of $e$ lies in $M$.
	After the execution of this subroutine, this matching is deleted from the graph, 
	hence the degree of each good edge decreases by 1.
	In particular, the maximum degree of good edges decrease by 1. 
	Once this maximum degree becomes zero, in the next iteration all good edges will be deleted.
	Therefore, after at most $K+m-1$ iterations, all good edges will be deleted, which means all $\nu$-good silver arms are explored, as required.
\end{proof}

\subsection{Proof of Theorem~\ref{thm:secondmain}(b).}
Theorem~\ref{thm:secondmain}(b) considers the stronger feedback model where  we observe the collision information but no lower bound $\mu$ for $\mu_m$ is known.
Note that in the algorithm for part (a), the parameter $\mu$ is mainly used to set the length of the MusicalChairs2 subroutines to make sure that each player will succeed in MusicalChairs2 with high probability.
For this part,  we observe the collision information, so we can modify MusicalChairs2 to use this information and determine its length without knowing $\mu$.

More precisely, the algorithm is the same as in part (a), 
except we set $\nu=0$ and
 $\alpha = 4 K \log (6 Km^2 T)$ and 
replace MusicalChairs2 with MusicalChairs3,  described next.
To obtain MusicalChairs3, we modify  MusicalChairs2  such that for a player to occupy an arm, she simply looks at the collision information and  occupies the arm if there is no collision. Its pseudocode appears below.

\begin{procedure}
	\SetAlgoLined
	\DontPrintSemicolon
	\KwIn{set $A\subseteq[K]$ of target arms, number $\alpha \in \{1,\dots\}$ of rounds}
	\KwOut{if an arm is occupied, the output is its index; otherwise, the output is 0}
	\For{$i\longleftarrow1$ \KwTo $\alpha$}
	{
		pull an arm $j\in [K]$ uniformly at random\;
		\If(\tcp*[h]{arm $j$ is occupied})
		{$j\in A$ \and there was no collision}
		{
			pull arm $j$ for the remaining $\alpha-i$ rounds\;
			output $j$\; 
		}
	}
	\tcp*[h]{no arm is occupied}\;
	output 0\;
	\caption{MusicalChairs3($A$, $\alpha$)}
\end{procedure}

The notions of success and failure are defined similarly as before but without a parameter $\nu$ (one can think  $\nu=0$ in this case: all arms are 0-good).
We have the following bound for its failure probability, whose statement and proof are identical to that for Corollary~\ref{cor:mcguarantee}, except there is no parameter $\nu$.

\begin{corollary}\label{cor:mcguaranteestronger} Let $\alpha$ be a positive integer.
	In the stronger feedback model with collision information available, the failure probability of MusicalChairs3 subroutine of length $\alpha$ is not more than 
	$m\exp(-\alpha/4K)$
	if $m\leq K$ and 
	$m\exp(-\alpha\exp(-2m/K)/K)$ in general.
\end{corollary}

The proof of Theorem~\ref{thm:secondmain}(b) is identical to part (a), except we use
Corollary~\ref{cor:mcguaranteestronger}
instead of 
Corollary~\ref{cor:mcguarantee};
we obtain the bound~\eqref{finalregretbound}, which using $\alpha = 4 K \log (6 Km^2 T)$ proves Theorem~\ref{thm:secondmain}(b).

\subsection{Proof of Theorem~\ref{thm:secondmain}(c).}\label{sec:c}
Part (c) considers the case that we do not know $\mu$ and we do not observe the collision information, but the players have the option to leave the game.
The trouble is that it is not clear how to choose the lengths of MusicalChairs2 subroutines.
To solve this issue, we choose really large lengths for 
MusicalChairs2 subroutines, and if a player has not occupied an arm at the end of a subroutine, she will leave the game.
This can  happen only if any remaining unoccupied arm has a really small mean, so we have not lost much by not pulling that arm anyway. We explain the details next.

We make the following changes to the algorithm for part (a):
we choose $\nu = K \log(T)/\sqrt{T}$
and define $\alpha,g$ using~\eqref{eqag} (so Lemma~\ref{lem:bad_event} still applies: all MusicalChairs2 subroutines are $\nu$-successful with high probability),
and we add the following line before Line~\ref{estimationLine}:
``\textbf{if} $j=0$ \textbf{then} leave the game.''
Namely, if a player has not occupied an arm when she wants to start an estimation period, she would simply leave the game and never pull any arm again.
Observe that this could happen only if there are fewer than $m$ many $\nu$-good arms,  so players may fail to find and occupy an arm.
Suppose $m'$ of the best $m$ arms are $\nu$-bad.
Once $m'$ players have left the game, we will have $m-m'$ players and $m-m'$ many $\nu$-good arms, so the algorithm will work as in part (a) from that point onward and the same analysis works. We would only lose a reward of at most $m' \nu T$ due to the players who have left the game.
The total expected regret can be thus bounded via~\eqref{finalregretbound} by
\begin{align*}
O(\nu m T + {K^2m\log^2 (T)}/{\nu} + Km\min \{\sqrt{T \log T}, \log(T)/\Delta'\})\\
=
O(mK\log(T) \sqrt{T} + Km\min \{\sqrt{T \log T}, \log(T)/\Delta'\}),
\end{align*}
completing the proof of Theorem~\ref{thm:secondmain}(c).

\section{Relaxing the assumptions.}
\label{sec:relaxing}
Recall that all the theorems presented so far made three assumptions:
\begin{description}
	\item{Assumption 1.}
	$K\geq m$: there are at least as many arms as players.
	\item{Assumption 2.}
	The rewards are supported on $[0, 1]$.
	\item{Assumption 3.}
	All players know the values of both $T$ and $m$.
\end{description}
Moreover, for different parts of Theorem~\ref{thm:secondmain} we have made additional assumptions.
In this section, we discuss how the Assumptions 1--3 can be removed  at the expense of getting worse regret bounds. Some assumptions can be removed independently of other assumptions, but some of them cannot be removed unconditionally; we discuss them one by one.

\subsection{Unknown time horizon.}
The assumption that $T$ is known can be removed independently of any other assumption, and the regret bound would  multiply by at most $\log_2(T)$.

Indeed, if $T$ is not known, we can apply a simple doubling trick (see~\cite{doubling_trick} for various variants):
we execute the algorithm assuming $T=1$,
then we execute it assuming $T=2 \times 1$, and so on,
until the actual time horizon is reached.
If the expected regret of the algorithm for a known time horizon $T$ is $R(T)$,
then the expected regret of the modified algorithm for an unknown time horizon would be 
\(R'(T)\leq\sum_{i=0}^{\lfloor \log_2 (T) \rfloor} R(2^i) \leq \log_2(T) \times R(T).\)
For example, if the players have the option of leaving the game,
we can apply Theorem~\ref{thm:secondmain}(c) 
to get the regret upper bound 
\[R'(T)\leq\sum_{i=0}^{\lfloor \log_2 (T) \rfloor}
O(Km\log(2^i) \sqrt{2^i}) \leq O(Km\log (T)) \times \sum_{i=0}^{\lfloor \log_2 (T) \rfloor} O(2^{i/2}) \leq O(Km\sqrt{T} \log (T)),\]
which is within a constant multiplicative factor of the upper bound for $R(T)$.

\subsection{Other reward distributions.}
The assumption that the rewards always lie in $[0,1]$ can be relaxed, independently of other assumptions, to the assumption that the rewards have subgaussian distributions with mean lying in a known interval; of course the regret bounds must be re-normalized, and we also get a multiplicative logarithmic factor in some cases.

In the proofs, we have used this assumption in three ways:
first, we used that the expected regret incurred any round can be bounded by $m$;
second, that the rewards satisfy the Chernoff-Hoeffding concentration inequality (Proposition~\ref{prop:concentration}(a));
and third, for bounding the failure probability of MusicalChairs2,3 subroutines we used  that $\p{X>0}\geq \E X$ for any random variable $X\in[0,1]$.

A random variable $X$ is {\em$\sigma$-sub-Gaussian} if 
$ \max\{\p{X-\E X <-t},\p{X-\E X >t}\} < \exp(-t^2/2\sigma^2)$;
for example, a standard normal random variable is $1$-sub-Gaussian.
The first two facts hold, with appropriate adjustments, for $\sigma$-sub-Gaussian random variables whose means lie in a bounded interval $[0,b]$, see, e.g.,~\cite[Chapter~2]{hdp-vershynin}.
The third fact also holds up to a logarithmic factor, see Lemma~\ref{lem:subgaussianpositivity} below.
Hence, after appropriate adjustments, our main theorems can be readily extended to such distributions.

\begin{lemma}
	\label{lem:subgaussianpositivity}
	Let $X\geq0$ be a random variable with mean $\mu$
	that satisfies $\p{X > \mu + t} < \exp(-t^2/2\sigma^2)$.
	Then we have
	$\p{X>0} \geq \min \{|\mu/(\sigma \log (\sigma/\mu))|,1 \}/99$.
\end{lemma}
\begin{proof}
	By dividing $X$ by $\sigma$ we may assume $\sigma=1$.
	Let $t\geq 0$ be a parameter to be chosen later, and define
	$Y = X \cdot \mathbf{1}[X > t+ \mu] $
	and
	$Z = X \cdot \mathbf{1}[X \leq t+ \mu] $.
	Note that $\mu = \E X = \E Y + \E Z$
	and 
	$\E Z \leq \p{X>0} (t+\mu)$.
	We next write $\E Y$ as
	\begin{align*}
	\E Y 
	& = \int_{0}^{t+\mu} \p{Y>s} \mathrm{d}s
	+ \int_{t+\mu}^{\infty} \p{Y>s} \mathrm{d}s 
	\end{align*}
	For $0\leq s \leq t+\mu$, we have $Y>s$ if and only if $Y>0$ if and only if $X > t+\mu$, whence
	$$\int_{0}^{t+\mu} \p{Y>s} \mathrm{d}s = (t+\mu)\p{X>t+\mu}
	< (t+\mu) \exp(-t^2/2).$$
	For the second integral, we have
	\[
	\int_{t+\mu}^{\infty} \p{Y>s} \mathrm{d}s 
	<
	\int_{t}^{\infty} \exp(-s^2/2) \mathrm{d}s 
	< \exp(-t^2/2)/2t.
	\]
	Consequently,
	\[
	\mu = \E Y + \E Z
	< 
	(t+\mu+1/2t) \exp(-t^2/2) + \p{X>0} (t+\mu),
	\]
	which implies
	\[
	\p{X>0} > \frac{\mu - (t+\mu+1/2t) \exp(-t^2/2)}{t+\mu}.
	\]
	Now, if $\mu\leq0.05$ then setting $t = \log(1/\mu)$ gives that the
	right-hand side is greater than $\mu / (5 \log(1/\mu))=|\mu / (5 \log(1/\mu))|$.
	(Here, we have used the inequality
	$$
	5 \mu \log(\mu) + 5 \log(\mu) \exp(-\log^2 \mu /2) (\log(\mu)-\mu+1/(2\log(\mu)))   -\mu \log(\mu)+\mu^2 < 0,$$ which holds for all $0<\mu\leq 0.05$.)
	
	On the other hand, if $\mu > 0.05$, setting $t=4$ gives that the
	right-hand side is greater than $1/99$, as required.
	(Here, we have used the inequality
	$(98-e^{-8})\mu > 4 + 33 \times e^{-8}/8$,
	which holds for any $\mu> 0.05$.)
\end{proof}

\subsection{More players than arms.}
We next consider the assumption that $K \geq m$ and explain how and when it can be removed.
First, note that if $K<m$ then the term $\sum_{i\in[m]} \mu_i$  in the definition of regret~(\ref{regret_def}) is not well defined, hence we must redefine the regret. There are two natural ways to do this.

\subsubsection{Original model.}
In the original model, if $K<m$, then the best strategy for the players,
had they known the means, would be for $K-1$ of them to occupy the best $K-1$ arms and for the rest to occupy the worst arm; so the regret in this case can be defined as
\[
\textnormal{Regret}=
T \sum_{i\in[K-1]} \mu_i
-
\sum_{t\in[T]}
\sum_{j\in[m]}
\mu_{A_j(t)} (1 - C_{A_j(t)}(t)).
\]
Let $\Delta \coloneqq \mu_{K-1}-\mu_K$.

For this model, we present an algorithm without observing the collision information and without the assumption $K\geq m$ 
with expected regret
$O(mK \log (T) \exp(4m/K) / \Delta^2)$.
We assume that $T$ is known and the rewards lie in $[0,1]$; we have explained in previous subsections how the regret bound will be affected if these are relaxed.
The algorithm crucially assumes $m$ is known to the players.

The algorithm is similar to Algorithm~\ref{alg1}.
Let $p \coloneqq (1-1/K)^{m-1} \geq \exp(-2m/K)$ be the probability of no-collision when the players pull arms uniformly at random, and let $g = C K \log(KT)/p^2$ for a sufficiently large constant $C$.
Each player pulls arms randomly until at some round $\tau$ she finds a gap of $3\sqrt{g/\tau}$ between the $(K-1)$th and the $K$th arm, and she continues for $24\tau$ more rounds to make sure that all others have also found this gap.
An argument similar to that of Lemma~\ref{cor:fitness} gives that these two phases will take $O(K \log (KT) / p^2 \Delta^2)$ many rounds.
Moreover, each player has learned that $\mu_{K-1}\geq\Delta \geq \sqrt{g/\tau}$ and that
$\sqrt{\tau/g}\leq 5/\Delta$
(see Lemma~\ref{cor:fitness}(ii)).
Then the player executes MusicalChairs2 on the 
set of $K-1$ best arms, for
$\alpha = C K \log (KT) \sqrt{\tau/g} / p$ many rounds, for a large enough constant $C$.
Since $m \exp(-\alpha \mu_{K-1} p / K)
\leq m \exp(-\alpha \sqrt{g/\tau} p / K) < 1/mT$,
Lemma~\ref{lem:mc2guarantee} implies that, with probability at least $1-1/mT$, all players will be $\sqrt{g/\tau}$-successful, meaning that the best $K-1$ arms are occupied.
After MusicalChairs2 is finished, if the player has occupied an arm, she will pull it until the end of game,
otherwise she pulls the worst arm for the rest of game.
Thus, the regret will be zero after at most
$O(K \log (KT) / p^2 \Delta^2)
+
O(K \log (KT) \sqrt{\tau/g} / p)
\leq 
O(K \log (KT) / p^2 \Delta^2)
$
many rounds, giving a total expected regret of 
$O(mK \log (KT) / p^2 \Delta^2)
\leq
O(mK \log (KT) \exp(4m/K) / \Delta^2)$.

\subsubsection{Model allowing players to leave.}
Alternatively, if we allow the players to leave the game,
the best strategy had they known the means would be for $m-K$ players to leave the game
and for the rest to occupy distinct arms.
The regret in this model can be defined as
\[
\textnormal{Regret}=
T \sum_{i\in[K]} \mu_i
-
\sum_{t\in[T]}
\sum_{j\in[m]}
\mu_{A_j(t)} (1 - C_{A_j(t)}(t))~.
\]
For this model, we present an algorithm without observing collision information and without the assumption $K\geq m$. 
We assume that $T$ is known and the rewards lie in $[0,1]$; we have explained in previous subsections how the regret bound will be affected if these are relaxed.
The algorithm crucially assumes $m$ is known to the players.

The algorithm is simple:
each player executes the MusicalChairs2 algorithm for a certain number of rounds,
and if she has not occupied an arm by that time, she leaves the game.

The number of rounds they play 
MusicalChairs2 is
$O(\log (TK) K \exp(2m/K) / \nu)$ with $\nu = \sqrt{m \log(TK) \exp(2m/K)/T}$.
With high probability, $\nu$-good arms  will be occupied,
and any other arm contributes a regret of at most $\nu T$.
So the total expected regret can be bounded by
$O(m \log (TK) K \exp(2m/K) / \nu + K \nu T)$,
which by the choice of $\nu$ gives the bound 
$O(K \exp(m/K) \sqrt {mT \log (TK)})$ for the expected regret.

If we make an additional assumption that the players know a lower bound $\mu_{\min}$ for all the arm means, then instead they play MusicalChairs2 for 
$O(\log (TK) K \exp(2m/K) / \mu_{\min})$
many rounds, and by Lemma~\ref{lem:mc2guarantee}, with probability at least $1-1/mT$, all
the $K$ arms are occupied, whence the total expected regret is bounded by
$O(m K \log (T) \exp(2m/K) / \mu_{\min})$.

Alternatively, if instead of knowing  $\mu_{\min}$  the players observe the collision information,
they play MusicalChairs3 for $O(\log (TK) K \exp(2m/K))$ many rounds,
and the total expected regret is upper bounded by
$O(mK \log (T)  \exp(2m/K))$.

\subsection{Unknown number of players.}\label{sec:unknown_players}
We next consider the assumption that $m$ is known and explain how it can be removed.
We assume that $T$ is known and the rewards lie in $[0,1]$; we have explained in previous subsections how the regret bound will be affected if these are relaxed.
Crucially, we assume $m \leq K$, although if $m \leq C K$ for some known absolute constant $C$, then the analysis in this section works after appropriate adjustments and all the derived asymptotic bounds hold.

In this section, we present two subroutines to estimate $m$ in two different settings: when the collision information is observed
and when the collision information is not observed but $\mu_1 \geq \mubar$ for some known $\mubar$.
If $m$ is unknown, such a subroutine can be executed at the beginning of the algorithm, and after that we can execute one of the algorithms presented previously;
hence the total regret bound would increase by the number of rounds of the subroutine times $m$.

In the first setting, when the players observe the collision information, \cite[Lemma~2]{musicalchair} presents a simple algorithm, with $O(K^2 \log(1/\delta))$ many rounds, using which each player learns $m$ with probability $\geq 1-\delta$.
Setting $\delta = 1/ K^2 T$ ensures that this simultaneously holds for all players with probability $\geq 1-1/KT$.  After this estimation, the players can run the algorithm of Theorem~\ref{thm:secondmain}(b).
The additional regret due to these estimation rounds is 
$O(K^2 m \log (KT))$, which is dominated by the final regret upper bound of Theorem~\ref{thm:secondmain}(b).

For the setting without the collision information, we assume that the players know that at least one arm has mean at least $\mubar$. We present an algorithm with $O(K^3 \log^2 (K/\mubar\delta) / \mubar^2)$ many rounds such that if all players execute it, each will learn $m$ with probability $1-\delta$.
Setting $\delta = 1/ K^2 T$ ensures that this simultaneously holds for all players with probability $\geq 1-1/KT$, and after this estimation, the players can execute Algorithm~1 or Algorithm~2.
The additional regret due to estimation is bounded by $O(K^3m \log^2 (KT/\mubar)/\mubar^2)$.

Here is the algorithm each player executes:
let
$\eps \coloneqq\mubar ((1-1/K)^{-2/5}-1)/48$ and 
observe that, since $K\geq m\geq 2$,
\begin{align*}
\eps & = \mubar/4 \times \frac 14 \times \frac 13 \times
((1-1/K)^{-2/5}-1) \\
&<
\mubar/4 \times (1-1/K)^{m-1} \times ((1-1/K)^{-2/5}+1)^{-1} \times ((1-1/K)^{-2/5}-1) \\
&=
\mubar/4 \times (1-1/K)^{m-1} \times 
\frac{(1-1/K)^{-2/5}-1}{(1-1/K)^{-2/5}+1}  \\
&<
\mubar/4 \times (1-1/K)^{m-1}.
\end{align*}
First, the player pulls random arms for $8 K \log (K^2/9\delta)/\eps^2$ rounds and estimates the quantities $\mu_j (1-1/K)^{m-1}$ for all $j\in[K]$.
By an argument similar to that of Lemma~\ref{lem:goodestimateofmeans}, she obtains  estimates $\{\sigma_j\}_{j\in[K]}$ such that 
\begin{equation}\label{cond}
	|\mu_j (1-1/K)^{m-1} - \sigma_j| \leq \eps \qquad \forall j\in[K]
\end{equation} 
for all players, uniformly with probability $1-\delta/3$.
Let $\ell$ be the arm with maximum $\sigma$ value. 
We claim that $\mu_{\ell} \geq \mubar/2$.
To prove this, note that
\[
(1-1/K)^{m-1} \mu_{\ell} \geq \sigma_{\ell}-\eps
\geq \sigma_{1}-\eps
\geq 
(1-1/K)^{m-1} \mu_{1} - 2\eps
\geq
(1-1/K)^{m-1} \mubar - 2\eps,
\]
whence
\(
\mu_{\ell} \geq \mubar - 2\eps/(1-1/K)^{m-1} \geq \mubar/2
\)
since $\eps \leq \mubar/4 \times (1-1/K)^{m-1}$.

Then the player tries to estimate $\mu_{\ell}$ itself and then uses the ratio $\mu_{\ell}/\sigma_{\ell}$ for estimating $m$.
For this, she tries $4K \log(6K/\mubar\delta)$ times to occupy the arm $\ell$, using a musical chairs subroutine:
divide the time horizon into $4K \log(6K/\mubar\delta)$ blocks of length $\log(6/\delta)/\eps^2$. For each block, she chooses an arm uniformly at random and pulls it for all the rounds in the block. If this arm was arm $\ell$ and she receives a positive reward at least once during the block, then, by taking the average of received rewards in the block, she obtains an unbiased estimate $\muhat$ for $\mu_{\ell}$. In any case, she repeats this procedure for the next blocks.
Using an analysis similar to that of MusicalChairs2,
after $4K \log (6K/\mubar\delta)$ iterations, with probability at least $1-\delta/3$, all players have explored their arm $\ell$.
The pseudocode appears in Algorithm~\ref{alg4} below.

\begin{algorithm}
	\caption{algorithm for estimating the number of players $m$}\label{alg4}
	\SetAlgoLined
	\DontPrintSemicolon
	\KwIn{number of arms $K$, 
		lower bound $\mubar$ on $\mu_1$, failure probability $\delta$}
	\KwOut{number of players $m$}
	$\eps \longleftarrow \mubar ((1-1/K)^{-2/5}-1)/48$\;
	\For{$8 K \log (K^2/9\delta)/\eps^2$ rounds}
	{pull a uniformly random arm $j$\;
		$\sigma_j \longleftarrow$
		average reward received from arm $j$\;
	}
	Let $\ell \longleftarrow \argmax_j \sigma_j$\;
	\For{$4K \log(6K/\mu\delta) $ iterations}
	{
		Pick arm $j$ uniformly at random\;
		Pull $j$ for $\log(6/\delta)/\eps^2$ times and
		let $\muhat \longleftarrow $ average reward received\;
		\If{$j=\ell$ and $\muhat>0$}
		{
			output $m$ satisfying
			$ [(\muhat - \eps)(1-1/K)^{m-1}, (\muhat + \eps)(1-1/K)^{m-1}]
			\cap [\sigma_{\ell} - \eps, \sigma_{\ell} + \eps] \neq \emptyset$
		}
	}
\end{algorithm}

For each player, since the estimate $\muhat$ is based on 
$\log(6/\delta)/\eps^2$ pulls, with probability $1-\delta/3$ she obtains an estimate $\muhat$ for $\mu_{\ell}$ such that $|\muhat-\mu_{\ell}| \leq \eps$.
Therefore, we have
$\mu_{\ell} \in [\muhat - \eps, \muhat + \eps]$
and also
$\mu_{\ell} (1-1/K)^{m-1} \in 
[\sigma_{\ell} - \eps, \sigma_{\ell} + \eps]$ by~\eqref{cond}.
Given the two intervals, we want to recover $m$.
Since 
$\eps < \mu/4 \times (1-1/K)^{m-1} \times 
\frac{(1-1/K)^{-2/5}-1}{(1-1/K)^{-2/5}+1}$,
we have
$$
\frac{\muhat+\eps}{\muhat-\eps}
\leq
\frac{\sigma_{\ell}+\eps}{\sigma_{\ell}-\eps}
\leq
\frac{\mu_{\ell}(1-1/K)^{m-1}+2\eps}{\mu_{\ell}(1-1/K)^{m-1}-2\eps}
\leq
\frac{\mubar(1-1/K)^{m-1}/2+2\eps}
{\mubar(1-1/K)^{m-1}/2-2\eps}
< (1-1/K)^{-2/5},
$$
hence Lemma~\ref{lem:uniquem} below shows that $m$ can be recovered uniquely. 

\begin{lemma}
	\label{lem:uniquem}
	Let $a,b,c,d,p > 0$.
	Consider intervals $[a,b]$ and $[c,d]$
	with $\max\{b/a,d/c\}\leq p^{2/5}$,
	and suppose there exist $x\in[a,b]$ and $y\in[c,d]$ such that
	$x p^z = y$ for some integer $z$.
	Then there exists a unique integer $n$ such that
	$ 
	[a p^n, b p^n] \cap [c, d] \neq \emptyset
	$.
\end{lemma}
\begin{proof}
	The existence of such an $n$ follows from existence of $x$ and $y$
	and that $x p^z = y$ for some integer $z$. For the uniqueness, note that we have
	$[a p^n, b p^n] \cap [c, d] \neq \emptyset
	$
	if and only if
	$[\log a/\log p + n , \log b/\log p + n] 
	\cap [\log c/\log p, \log d/\log p] \neq \emptyset
	$.
	Now note that the interval $[\log c/\log p, \log d/\log p]$ has length $\leq 2/5$.
	Each interval $I_n=[\log a/\log p + n , \log b/\log p + n]$ also has length $\leq 2/5$, 
	hence, for each $n$,  
	$I_n$ and $I_{n+1}$ are at least $3/5$ apart from each other, so 
	$[\log c/\log p, \log d/\log p]$ can intersect at most one $I_n$.
\end{proof}

To bound the number of rounds of the algorithm, note that
\[
(1-1/K)^{-2/5}-1
=
(1+\frac{1}{K-1})^{2/5}-1
\geq
1 + \frac{2/5}{K-1} -1 = \frac{2}{5(K-1)} > \frac{2}{5K},
\]
thus $\eps \geq \mubar/120K$. So the number of rounds of the algorithm is 
$
8 K \log (K^2/9\delta)/\eps^2
+
4K \log (6K/\mubar\delta)
\log(6/\delta)/\eps^2
=O(K^3 \log^2 (K/\mubar\delta) / \mubar^2)$, as required.

\section{Proof of Theorem~\ref{thm:anticoordination}.}
\label{sec:thirdproof}
In this section, we present a distributed algorithm that, with probability at least $1-\delta$, converges to an $\eps$-Nash equilibrium in any stochastic anti-coordination game within
$O(\log(K/\delta) (K/\eps^2+K^2/\eps))$ many rounds.

Note that the players do not observe collisions, and in particular, they do not observe the actions of other players, but we assume each player has the option of choosing a dummy action, which is given index 0 and produces no reward.
We are still making the Assumptions 1--3 stated on page~\pageref{assumptions} (but there is no parameter $T$ here).

We describe the algorithm each player executes.
First, the player pulls arms uniformly at random and maintains an estimate for the arm means.
An argument similar to that of Lemma~\ref{lem:goodestimateofmeans} gives that,
after $512K \log (6mK/\delta) /\eps^2$ rounds,
with probability at least $1-\delta/2m$, all estimated means are within distance $\eps/2$ of the actual means.
By a union bound over all players, this is true uniformly over all players with probability at least $1-\delta/2$.

The player then sorts the $\muhat_i$ as
$\muhat_{i_1}\geq \dots \geq \muhat_{i_K}$.
Then for $j=1,\dots,K$, she runs MusicalChairs2 on $\{i_j\}$ (in this order) for $4 K \log(2mK/\delta) / \eps$ many rounds. 
If during any of these subroutines she occupies an arm, she chooses that action.
Otherwise, she chooses the dummy action 0.
The pseudocode is given as Algorithm~\ref{alg3}.

\begin{algorithm}
	\caption{algorithm for reaching an $\eps$-approximate Nash Equilibrium in an anti-coordination game}\label{alg3}
	\SetAlgoLined
	\DontPrintSemicolon
	\KwIn{number of players $m$, 
		number of arms $K$, 
		accuracy $\eps$,
		failure probability $\delta$}
	\KwOut{action $\ell$}
	$\muhat_i \longleftarrow 0$ for all $i\in[K]$\;
	\For{$512K \log (6mK/\delta) /\eps^2$ rounds}
	{pull a uniformly random arm $j$\;
		$\muhat_j \longleftarrow$
		average reward received from arm $j$ divided by $(1-1/K)^{m-1}$\;
	}
	Sort the $\mathbf {\muhat}$ vector as 
	$\muhat_{i_1}\geq \dots \geq \muhat_{i_K}$\;
	$\ell \longleftarrow 0$\;
	\For{$j=1$ \KwTo $K$}
	{
		$\ell \longleftarrow $ MusicalChairs2 ($\{i_j\}$, $4 K \log(2mK/\delta) / \eps$)\;
		\lIf{$\ell\neq0$}{pull arm $\ell$ for the remaining 
			 rounds}\nllabel{remainingrounds}
	}
	Output $\ell$\;
\end{algorithm}

By Corollary~\ref{cor:mcguarantee} and a union bound over the $K$ iterations, all the MusicalChairs2 subroutines for all players are $\eps$-successful with probability at least $1-\delta/2$.
We now show that if the estimation errors are $\leq \eps/2$
and all the MusicalChairs2 subroutines are $\eps$-successful (with probability $1-\delta$ both these good events happen), then the resulting assignment is an $\eps$-Nash Equilibrium. 
Fix any player $p$ and recall that, for each action $i\in[K]$, $\mu_i^p$ denotes the average reward player $p$ would receive if she plays action $i$ solely.
First, suppose that she has output a non-dummy action $i_j$.
This means all actions $i_1,i_2,\dots,i_{j-1}$ were either occupied by other players or had mean $<\eps$ or both.
On the other hand, since the estimated means are within $\eps/2$ of
the actual means,
for any $s \notin \{ i_1,i_2,\dots,i_{j-1} \}$ we have
$\muhat_{s} \leq \muhat_{i_j}$ so
\[\mu_s^p =
(\mu_s^p - \muhat_s) + (\muhat_{s} - \mu_{i_j}^p) + \mu_{i_j}^p 
\leq 
(\mu_s^p - \muhat_s) + (\muhat_{i_j} - \mu_{i_j}^p) + \mu_{i_j}^p
\leq
\eps/2+\eps/2+\mu_{i_j}^p,\] 
hence the player cannot increase her outcome by more than $\eps$ by switching to action $s$.
Finally, if player $p$ has chosen the dummy action $0$, it means that for each $j\in[K]$,
either action $i_j$ is occupied or $\mu_{j}^p \leq \eps$ or both.
Thus, there is no unoccupied action $s$ with $\mu_{s}^p > \eps$, so again the player cannot increase her outcome by more than $\eps$ by switching.

The total number of rounds is 
$$512K \log (6mK/\delta) /\eps^2+K\times 4 K \log(2mK/\delta) / \eps=
O(K\log(K/\delta)/\eps^2 + K^2\log(K/\delta)/\eps),$$ 
and the failure probability is at most $\delta$, as required.

\section*{Acknowledgments.}
We thank the referees of \textit{Mathematics of Operations Research} for detailed feedback, which resulted in significant improvements in the presentation.
G\'abor Lugosi was supported by the Spanish Ministry of Economy and Competitiveness,
Grant PGC2018-101643-B-I00 ``Predicc\'on, inferencia y computaci\'on en
modelos estructurados -
Ayudas Fundaci\'on BBVA a Equipos de Investigaci\'on Cientifica 2017'' 
and by ``Google Focused Award Algorithms and Learning for AI.''
Abbas Mehrabian was supported by an IVADO-Apog\'ee-CFREF postdoctoral fellowship.
This work started during the {Mathematics of Machine Learning} program
sponsored by the Centre de Recherches Math\'ematiques (CRM)
held at Universit\'e de Montr\'eal in April 2018.

\end{document}